%% file: main.tex
\def\isarxiv{1} 

\ifdefined\isarxiv
\documentclass[11pt]{article}

\usepackage[numbers]{natbib}

\else
\documentclass{article}
\usepackage{microtype}
\usepackage{graphicx}
\usepackage{subfigure} 
\usepackage{hyperref}
\usepackage{icml2024} 
\fi

\usepackage{amsmath}
\usepackage{amsthm}
\usepackage{amssymb}
\usepackage{algorithm}
\usepackage{algpseudocode}
\usepackage{grffile}
\usepackage{wrapfig,epsfig}
\usepackage{url}
\usepackage{xcolor}
\usepackage{epstopdf}

\usepackage{bbm}
\usepackage{dsfont}

\allowdisplaybreaks

\ifdefined\isarxiv

\usepackage{tikz}
\usepackage{hyperref}  
\usepackage{booktabs}
\hypersetup{colorlinks=true,citecolor=blue,linkcolor=blue} 
\usetikzlibrary{arrows}
\usepackage[margin=1in]{geometry}

\else

\usepackage{booktabs}
\usepackage{microtype}
\usepackage{hyperref}

\fi

\newtheorem{theorem}{Theorem}[section]
\newtheorem{lemma}[theorem]{Lemma}
\newtheorem{definition}[theorem]{Definition}

\newtheorem{fact}[theorem]{Fact}

\newcommand{\wh}{\widehat}

\newcommand{\R}{\mathbb{R}}

\renewcommand{\d}{\mathrm{d}}

\renewcommand{\hat}{\wh}

\DeclareMathOperator*{\E}{{\mathbb{E}}}

\makeatletter
\newcommand*{\RN}[1]{\expandafter\@slowromancap\romannumeral #1@}
\makeatother

\usepackage{lineno}

\begin{document}

\ifdefined\isarxiv

\date{}

\title{Enhancing Stochastic Gradient Descent: A Unified Framework and Novel Acceleration Methods for Faster Convergence}
\author{
Yichuan Deng\thanks{\texttt{ycdeng@cs.washington.edu}. The University of Washington.}
\and
Zhao Song\thanks{\texttt{zsong@adobe.com}. Adobe Research.}
\and 
Chiwun Yang\thanks{\texttt{christiannyang37@gmail.com}. 
Sun Yat-sen University.}
}

\else


\twocolumn[

\icmltitle{Enhancing Stochastic Gradient Descent:\\ A Unified Framework and Novel Acceleration Methods for Faster Convergence}


\icmlsetsymbol{equal}{*}

\begin{icmlauthorlist}
\icmlauthor{Aeiau Zzzz}{equal,to}
\icmlauthor{Bauiu C.~Yyyy}{equal,to,goo}
\icmlauthor{Cieua Vvvvv}{goo}
\icmlauthor{Iaesut Saoeu}{ed}
\icmlauthor{Fiuea Rrrr}{to}
\icmlauthor{Tateu H.~Yasehe}{ed,to,goo}
\icmlauthor{Aaoeu Iasoh}{goo}
\icmlauthor{Buiui Eueu}{ed}
\icmlauthor{Aeuia Zzzz}{ed}
\icmlauthor{Bieea C.~Yyyy}{to,goo}
\icmlauthor{Teoau Xxxx}{ed}
\icmlauthor{Eee Pppp}{ed}
\end{icmlauthorlist}

\icmlaffiliation{to}{Department of Computation, University of Torontoland, Torontoland, Canada}
\icmlaffiliation{goo}{Googol ShallowMind, New London, Michigan, USA}
\icmlaffiliation{ed}{School of Computation, University of Edenborrow, Edenborrow, United Kingdom}

\icmlcorrespondingauthor{Cieua Vvvvv}{c.vvvvv@googol.com}
\icmlcorrespondingauthor{Eee Pppp}{ep@eden.co.uk}

\icmlkeywords{Machine Learning, ICML}

\vskip 0.3in
]

\printAffiliationsAndNotice{\icmlEqualContribution}
\fi

\ifdefined\isarxiv
\begin{titlepage}
  \maketitle
  \begin{abstract}
\input{abstract}

  \end{abstract}
  \thispagestyle{empty}
\end{titlepage}

{\hypersetup{linkcolor=black}
\tableofcontents
}
\newpage

\else

\begin{abstract}
\input{abstract}
\end{abstract}

\fi

\input{intro}

\input{related_work}

\input{method}
\input{experiments}

\input{discussion}

\input{conclusion}

\ifdefined\isarxiv
\else
\input{impact_statement}
\fi

\ifdefined\isarxiv
\else

\bibliographystyle{icml2024}
\bibliography{ref}
\fi

\newpage
\onecolumn
\appendix
\section*{Appendix}
\input{roadmap}

\input{proofs}

\input{Adam_RVA}
\newpage

\ifdefined\isarxiv

\bibliographystyle{alpha}
\bibliography{ref}
\else

\fi




\end{document}

%% file: abstract.tex
Based on SGD, previous works have proposed many algorithms that have improved convergence speed and generalization in stochastic optimization, such as SGDm, AdaGrad, Adam, etc. However, their convergence analysis under non-convex conditions is challenging. In this work, we propose a unified framework to address this issue. For any first-order methods, we interpret the updated direction $g_t$ as the sum of the stochastic subgradient $\nabla f_t(x_t)$ and an additional acceleration term $\frac{2|\langle v_t, \nabla f_t(x_t) \rangle|}{\|v_t\|_2^2} v_t$, thus we can discuss the convergence by analyzing $\langle v_t, \nabla f_t(x_t) \rangle$. Through our framework, we have discovered two plug-and-play acceleration methods: {\bf Reject Accelerating} and {\bf Random Vector Accelerating}, we theoretically demonstrate that these two methods can directly lead to an improvement in convergence rate.

%% file: intro.tex
\section{Introduction}\label{sec:intro}

\paragraph{Gradient Descent and Stochastic Gradient Descent.}
Gradient Descent is a pivotal optimization algorithm in machine learning and neural networks \cite{rhw86}, used for minimizing a function, typically the loss function of a model, denoted as $ f(\theta) $. The essence of this method lies in iteratively adjusting the parameters $ \theta $ of the model in the opposite direction of the gradient of the function $ \nabla f(\theta) $ at the current point. Mathematically, this is represented as 
\begin{align*}
    \theta_{\text{new}} = \theta_{\text{old}} - \eta \nabla f(\theta),
\end{align*}
where $ \eta $ is the learning rate, a crucial hyperparameter that determines the size of the steps taken towards the minimum.

Building on the foundation of Gradient Descent, Stochastic Gradient Descent (SGD) introduces a significant variation. Instead of calculating the gradient of the entire dataset for updating the model's parameters, SGD computes the gradient using a randomly selected subset of the data, typically a mini-batch, in each iteration. This approach is described by the equation 
\begin{align*}
    \theta_{\text{new}} = \theta_{\text{old}} - \eta \nabla f_i(\theta),
\end{align*}
where $ f_i $ indicates the loss function calculated on the $i$-th subset of the data. This stochastic nature of SGD not only makes it computationally more efficient, especially for large datasets, but also helps in escaping local minima, making it more likely to find global minima in complex optimization landscapes.

\paragraph{Accelerating the SGD.}
Based on the basic SGD, there have been a lot of approaches to accelerate the descent process \cite{nes83, rhw86, pj92, dhs11, smdh13, Sut13, kb14, rkk18}. 

SGD with momentum \cite{nes83, rhw86, tse98, q99} is an enhanced version of the basic SGD algorithm, accelerating convergence by considering past gradients to inform the next step's direction. Its update rule involves a momentum term $ v $ and is mathematically expressed as 
\begin{align*}
    &~v_{\text{new}} = \gamma v_{\text{old}} + \eta \nabla f_i(\theta) \\
    \text{and}&~ \theta_{\text{new}} = \theta_{\text{old}} - v_{\text{new}},
\end{align*}
where $ \gamma $ is the momentum coefficient, $ \eta $ is the learning rate, and $ \nabla f_i(\theta) $ is the gradient of the loss function. This approach smooths updates and can expedite convergence in complex loss surfaces, typical in deep learning. Averaged SGD \cite{pj92} modifies the standard SGD by updating the parameter $\theta$ as an average of its values over all iterations, mathematically expressed as 
\begin{align*}
    \bar{\theta} = \frac{1}{T}\sum_{t=1}^{T}\theta_t,
\end{align*}
enhancing generalization. AdaGrad \cite{dhs11} adjusts the learning rate for each parameter $\theta_i$ by dividing it with the square root of the sum of squared gradients, given by 
\begin{align*}
    \theta_{i, \text{new}} = \theta_{i, \text{old}} - \frac{\eta}{\sqrt{G_{ii} + \epsilon}} \nabla f_i(\theta),
\end{align*}
where $G$ is a diagonal matrix where each diagonal element $i,i$ is the sum of the squares of the gradients w.r.t. $\theta_i$ up to time step $t$, and $\epsilon$ is a small smoothing term to avoid division by zero. Adam \cite{kb14, rkk18} combines momentum and RMSprop \cite{hss12}, updating parameters $\theta$ using two moving averages, the first moment $m_t$ and the second moment $v_t$, with the update rules
\begin{align*}
    &~m\beta_1 m_{t-1} + (1 - \beta_1) \nabla f_t(\theta) \\
    \text{and}&~v_t = \beta_2 v_{t-1} + (1 - \beta_2) \nabla f_t(\theta)^2,
\end{align*}
and the parameter update $\theta_{\text{new}} = \theta_{\text{old}} - \frac{\eta}{\sqrt{v_t} + \epsilon} m_t$, where $\beta_1$ and $\beta_2$ are the decay rates for the moment estimates. This approach allows Adam to adapt the learning rate based on both the average first moment (the mean) and the second moment (the uncentered variance) of the gradients, offering efficient and effective optimization, particularly in deep learning applications with large datasets and high-dimensional parameter spaces. 

\paragraph{A Unified Framework.}
For all of these accelerating methods invented, it is natural to ask the following question
\begin{center}
{\it 
    Can we summarize all the methods in a unified framework? 
}   
\end{center}
To answer this, we introduce the first unified framework that provides a convergence explanation of all the accelerated SGD methods. In our perspective, the update rule of any first-order stochastic accelerating algorithms can be written as $\theta_{t+1} = \theta_t - \eta \cdot g_t$, where $g$ denotes the direction of the update. Next, we consider that $g_t$ is an SGD-based direction that can be decomposed as:
\begin{align*}
    g_t = \nabla f_t(\theta_t) + \eta_t v_t
\end{align*}
where $v_t$ is an additional acceleration term and $\eta_t$ is an adaptive coefficient that $\eta_t = 2|\langle v_t, \nabla f_t(\theta_t) \rangle| / \|v_t\|^2$.

Moreover, let $k, l, u_a, u_b, B > 0$ be defined as Definition~\ref{def:k:informal}, Definition~\ref{def:l:informal}, Definition~\ref{def:u:informal} and Definition~\ref{def:B:informal}, we show that the convergence rate of any first-order stochastic optimization method can be written as:
\begin{align*}
    \frac{\sqrt{T+8kB}}{T+2ku_a-2lu_b}.
\end{align*}
Meanwhile, the the convergence rate of SGD is $\frac{1}{\sqrt{T}}$ \cite{rhs+16}.

\paragraph{Fast-ever SGD with Accelerating.}

In this study, we significantly advance the understanding of stochastic optimization processes, particularly under the complex dynamics of non-convex conditions. This advancement is grounded in the core result of our comprehensive theoretical framework. Leveraging this framework, we have innovatively developed two new methods that can be seamlessly integrated into existing first-order optimization techniques. These methods, which we have named ``Reject Accelerating'' and ``Random Vector Accelerating,'' are designed to significantly enhance the efficiency and speed of optimization processes.

The first method, ``Reject Accelerating,'' offers a novel approach to optimizing first-order acceleration methods. It functions by meticulously analyzing the correlation between the stochastic subgradient, denoted as $\nabla f_t(x_t)$, and an auxiliary accelerating term, $v_t$. These relationships are extensively defined in Definitions \ref{def:k:informal} and \ref{def:l:informal} of our framework. A key feature of this method is its ability to discern and react to inconsistencies between $v_t$ and $\nabla f_t(x_t)$. Specifically, if the inner product $\langle v_t, \nabla f_t(\theta_t) \rangle$ is less than or equal to zero, indicating inconsistency, the method strategically chooses not to employ $v_t$ for acceleration purposes.

The second method, ``Random Vector Accelerating,'' utilizes a different approach by employing Gaussian vectors, specifically $v_t \sim \mathcal{N}(0, \mathbf{I}_d)$, to expedite the convergence rate of SGD. This approach capitalizes on the statistical properties of Gaussian vectors to significantly improve the speed and efficiency of the stochastic optimization process. Our theoretical analysis, grounded in the framework, supports the effectiveness of this method in achieving faster convergence rates with a higher degree of reliability and predictability.

Both methods are a testament to the versatility and robustness of our unified framework in addressing the complexities of non-convex stochastic optimization. They not only contribute to the theoretical understanding of such processes but also offer practical tools for enhancing the efficiency of optimization algorithms in various applications.

%% file: related_work.tex
\section{Related Work}

In this section, we briefly review the prior works that have a close connection to our work. Specifically, we divide them into three parts using two topics: Stochastic Gradient Descent and Applications in Machine Learning and Stochastic Optimization.

\paragraph{Stochastic Gradient Descent and Applications in Machine Learning.}
SGD \cite{rhw86} and its variety, including SGD with momentum \cite{nes83, rhw86, smdh13, Sut13}, Averaged SGD \cite{pj92}, AdaGrad (adaptive gradient algorithm) \cite{dhs11} and Adam (adaptive moment estimation) \cite{kb14, rkk18}, have a long history in Machine Learning and Deep Learning. These optimization methods have promoted the development of Machine Learning and Artificial Intelligence, e.g. the explosive development of Large Language Models (LLMs) \cite{rns+18, dclt18, rwc+19, bmr+20, cnd+22, zrg+22, cha22, o23}. The SGD methods have also been used for the theoretical understanding and convergence analysis of LLMs \cite{dls23, dms23, dlms23, dm23, dsx23, dsxy23, csy23, csy23b}. 

\paragraph{Stochastic Optimization.}
Together with the prosperity of SGD, there have been a lot of works on stochastic optimization. \cite{gl13} presented the Randomized Stochastic Gradient (RSG) method, for efficiently solving nonlinear stochastic programming problems. \cite{rhs+16} analyzed the Stochastic Variance Reduced Gradient (SVRG) method for nonconvex optimization, proving its faster convergence compared to SGD and gradient descent, and demonstrating its linear convergence to global optima for certain nonconvex problems, including mini-batch variants in parallel settings. \cite{azh16} introduced a novel variance reduction-based first-order minibatch stochastic method for non-convex optimization, achieving faster convergence rates than previous methods and demonstrating its effectiveness in empirical risk minimizations and neural network training. \cite{zrs+18} analyzed adaptive gradient methods like RMSProp and Adam in nonconvex optimization, showing that increasing minibatch sizes ensures convergence, and introduced a new algorithm, Yogi, which effectively controls the learning rate and outperforms existing methods in various machine learning tasks. \cite{az18} developed two algorithms, SGD3 for convex functions with near-optimal convergence rates, and SGD5 for nonconvex functions, offering significant improvements over traditional SGD methods and matching the performance of the best stochastic Newton's method. \cite{dm23} introduced a single-loop, parameter-free gradient descent method for convex functions that achieves optimal convergence without prior knowledge of certain problem-specific parameters, and it validated the method with extensive experiments showing its efficacy on various large-scale machine learning tasks.

%% file: method.tex
\section{Problem Definition}\label{sec:problem_def}

In this section, we formally define the main goal to solve in this paper and provide the assumptions of theoretical analysis we use in this paper. We consider the non-convex optimization problem of neural networks in real-world cases, that is:
\begin{align}\label{eq:problem}
    \min_x f(x) := \frac{1}{n} \sum_{i=1}^n f_i(x).
\end{align}
Each $f_i$ for $i \in [n]$ is used to denote the loss function of a neural network on disjoint batch data. The $f$ is not necessarily convex, but its gradient should be Lipschitz smooth. We provide the definition as follows:
\begin{definition}
    We say $f: \R^d \rightarrow \R$ is {\it $L$-smooth} such that
\begin{align*}
    \| \nabla f(x) - \nabla f(y) \|_2 \leq L \| x - y \|_2^2, \forall x, y \in \R^d.
\end{align*}
\end{definition}

Since an objective $f$ that is {\it $L$-smooth}, an inequality commonly used in first-order optimization analysis is Fact~\ref{fac:lipschitz:informal} we provide below.
\begin{fact}\label{fac:lipschitz:informal}
    If $f: \R^d \rightarrow \R$ is {\it $L$-smooth}, then for any $x, y \in \R^d$, there is
    \begin{align*}
        f(y) \leq f(x) + \langle \nabla f(x), y - x \rangle + \frac{L}{2} \| x - y \|_2^2
    \end{align*}
\end{fact}
Next, we further assume that all gradients of $f_i$ for $i \in [n]$ are {\it $\sigma$-bounded}.
\begin{definition}
    We say $f: \R^d \rightarrow \R$ has {\it $\sigma$-bounded gradients} if $\| \nabla f_i(x) \|_2 \leq \sigma$, for all $x \in \R^d, i \in [n]$.
\end{definition}

The definitions provided above encapsulate the core problem and conditions in optimizing neural networks. To address Equation \eqref{eq:problem}, we commence by setting the initial weight of parameters as $x_0 \in \R^d$. Subsequently, we define $T$ as the total number of iterations in stochastic methods, where $T$ is an integer greater than zero. The weight of parameters at the $t$-th iteration is represented as $x_t \in \R^d$, where $t$ is an integer within the range of $[T]$.

Thus, we denote $x^* = \min_{x\in \R^d} f(x)$ as the optimal solution. In our approach, we evaluate the performance of first-order stochastic algorithms by using the concept of convergence rate. We define the convergence rate as follows:
\begin{definition}[Convergence rate]
    For constant $r > 0$, if a first-order stochastic algorithm ran $T>0$ and outputs the minimum value of the square of the norm of the gradient that satisfies:
    \begin{align*}
        \min_{t \in [T]} \| \nabla f(x) \|_2^2 \leq r \sqrt{2(f(x_0) - f(x^*))L\sigma^2}
    \end{align*}
    we refer $r$ as the convergence rate of this algorithm.
\end{definition}
A lower convergence rate indicates faster convergence towards the optimal solution.

\section{A Universal Convergence Analysis Framework of Accelerating Algorithms}\label{sec:framework}

In this section, we present a universal framework for analyzing stochastic first-order optimization. Initially, we elucidate all first-order algorithms in relation to  $g_t = \nabla f_t(x_t) + \frac{|\langle \nabla f_x(x_t), v_t\rangle|}{\|v_t\|_2^2}v_t$ (Definition~\ref{def:v:informal}). Subsequently, we delve into the scenarios where $\langle \nabla f_x(x_t), v_t\rangle > 0$ and $\langle \nabla f_x(x_t), v_t\rangle \leq 0$ in Section~\ref{sub:case1} and Section~\ref{sub:case2} respectively. In Section~\ref{sub:main_result}, we demonstrate the theoretical results of our universal framework.

The most prevalent method employed in real-world scenarios to solve the primary objective $f(x)$, as defined in Section~\ref{sec:problem_def}, is the SGD. However, SGD has seen significant advancements in terms of convergence speed and generalization. Previous studies have investigated numerous acceleration algorithms to tackle this issue, including SGDm, Adam, RMSprop, AdaGrad, among others. Nevertheless, analyzing these accelerating methods poses a challenge, particularly under non-convex conditions. In our methodology, we can transform the implementation of any first-order stochastic algorithm. At the $t$-th optimization step, we have
$
    x_{t+1} = x_t - \alpha g_t,
$
where $\alpha$ is a fixed step-size, and $g_t$ is the direction of the update in optimization. For instance, in SGD, $g_t = \nabla f_t(x_t)$; in Adam, $g_t = \frac{\hat{m_t}}{\sqrt{\hat{v_t}}+\epsilon}$.

Subsequently, we break down $g_t$ into a subgradient of $f(x_t)$ and an additional accelerating term $\eta_t v_t$, where $\eta_t = \frac{2 |\langle f_t(x_t), v_t \rangle|}{| v_t |2^2}$. We provide the definition of $v_t$ below.
\begin{definition}[Additional accelerating term $v_t$]\label{def:v:informal}
    For any optimization algorithm that can be described as $x_{t+1} = x_t - \alpha g_t$. We define the {\it additional accelerating term} $v_t$ that satisfies $\frac{2 |\langle f_t(x_t), v_t \rangle|}{\| v_t \|_2^2} v_t := g_t - \nabla f_t(x_t)$.
\end{definition}

\paragraph{Consistency between \texorpdfstring{$\nabla f_t(x_t)$}{} and \texorpdfstring{$v_t$}{}.}
In our framework, the concept of consistency between the stochastic subgradient $\nabla f_t(x_t)$ and the additional accelerating term $v_t$ is crucial. We define consistency as the condition where $\langle \nabla f_t(x_t), v_t \rangle > 0$, indicating that $\nabla f_t(x_t)$ and $v_t$ are aligned in the same direction. Conversely, if $\langle \nabla f_t(x_t), v_t \rangle \leq 0$, they are considered inconsistent. We will prove that ensuring consistency between $\nabla f_t(x_t)$ and $v_t$ is essential for achieving convergence during the optimization process.

To describe the situations when consistency or inconsistency occurs, we introduce the following definitions:
\begin{definition}\label{def:gamma:informal}
    We define
    \begin{align*}
        & ~ \gamma_k(t) := \left\{
            \begin{array}{lr}  
             1, & \langle v_t, \nabla f_t(x_t) \rangle > 0 \\  
             0, & \text{\rm otherwise}
             \end{array} 
        \right. \\
        & ~ \gamma_l(t) := \left\{
            \begin{array}{lr}  
             1, & \langle v_t, \nabla f_t(x_t) \rangle \leq 0 \textrm{ and } \|v_t\|_2 \ne 0 \\  
             0, & \text{\rm otherwise}
             \end{array} 
        \right.
    \end{align*}
\end{definition}

We use positive integers $k, l$ to tally the occurrences of $\langle v_t, \nabla f_t(x_t) \rangle > 0$ and $\langle v_t, \nabla f_t(x_t) \rangle \leq 0$, respectively, leading to:
\begin{definition}\label{def:k:informal}
    For an optimization algorithm that has been run $T$ times where integer $k > 0$, let $\gamma_k(t)$ be defined as Definition~\ref{def:gamma:informal}, we define $k := \sum_{t=0}^{T-1} \gamma_k(t)$.
\end{definition}

\begin{definition}\label{def:l:informal}
    For an optimization algorithm that has been run $T$ times where integer $l > 0$, let $\gamma_l(t)$ be defined as Definition~\ref{def:gamma:informal}, we define $l := \sum_{t=0}^{T-1} \gamma_l(t)$.
\end{definition}

\paragraph{Expectations of \texorpdfstring{$\frac{\langle \nabla f(x_t), v_t \rangle^2}{\| v_t \|_2^2}$}{} and \texorpdfstring{$\frac{\langle \nabla f_t(x_t), v_t \rangle^2}{\| v_t \|_2^2}$}{}.}
Two important variables in our framework are the expectations of $\frac{\langle \nabla f(x_t), v_t \rangle^2}{\| v_t \|_2^2}$ and $\frac{\langle \nabla f_t(x_t), v_t \rangle^2}{\| v_t \|_2^2}$. We can derive these expectations as follows:
\begin{align*}
    \E[\frac{\langle \nabla f(x_t), v_t \rangle^2}{\| v_t \|_2^2}] = 
    \E[\|\nabla f(x_t)\|_2^2 \cos^2{\theta}]
\end{align*}
where $\theta = \frac{\langle \nabla f(x_t), v_t \rangle}{\| f(x_t) \|_2\|v_t\|_2}$. In our approach, we transform this formula into a linear relationship as $\E[\frac{\langle \nabla f(x_t), v_t \rangle^2}{\| v_t \|_2^2}] = k \cdot \E[ \| \nabla f(x_t)\|_2^2]$, then in Definition~\ref{def:u:informal}, we further loose this relationship to a range by defining $k \in [u_a, u_b]$ for constants $u_a, u_b > 0$ to facilitate our analysis.
\begin{definition}\label{def:u:informal}
    We define $u_a \E[ \| \nabla f(x_t) \|_2^2 ] \leq \E[ \frac{\langle \nabla f(x_t), v_t \rangle^2}{\| v_t \|_2^2} ] \leq u_b \E[ \| \nabla f(x_t) \|_2^2 ]$ where $u_a, u_b > 0$ are constants.
\end{definition}

Similarly, we belive that the above relationship also exists between $\E[\frac{\langle \nabla f_t(x_t), v_t \rangle^2}{\| v_t \|_2^2}]$ and $\E[\| \nabla f_t(x_t) \|_2^2]$ (Definition~\ref{def:B:informal}).
\begin{definition}\label{def:B:informal}
    We define $\E[ \frac{\langle \nabla f_t(x_t), v_t \rangle^2}{\| v_t \|_2^2} ] \leq B \E[ \| \nabla f_t(x_t) \|_2^2 ]$ where $B > 0$ is denoted as a constant.
\end{definition}

\subsection{In the Case of \texorpdfstring{$\langle v_t, f_t(x_t) \rangle > 0$}{}}\label{sub:case1}

In this section, we discuss the inequality of $\E[f(x_{t+1})]$ and $\E[f(x_t)]$ under the condition of $\langle v_t, f_t(x_t) \rangle > 0$. We provide the lemma as follows: 
\begin{lemma}[Informal version of Lemma~\ref{lem:case1}]\label{lem:case1:informal}
    $f: \R^d \rightarrow \R$ is a $L$-smooth function and has $\sigma$-bounded gradients. Denote $T > 0$ as a positive integer, for a stochastic iterative first-order optimization algorithm that implements $T$ times. For $t \in [T]$, we follows Definition~\ref{def:v:informal} to write it as $x_{t+1} = x_t - \alpha g_t = x_t - \alpha ( \frac{2 |\langle f_t(x_t), v_t \rangle|}{\| v_t \|_2^2} v_t + \nabla f_t(x_t) )$. Let $u_a, u_b > 0$ be defined as Definition~\ref{def:u:informal}. Let $B > 0$ be defined as Definition~\ref{def:B:informal}. If $\langle v_t, f_t(x_t) \rangle > 0$, we have
    \begin{align*}
        \E[ f(x_{t+1}) ] \leq & ~ \E[ f(x_t) ] - \alpha ( 1 + 2u_a ) \E[ \| \nabla f(x_t) \|_2^2 ] \\
        & ~ + (4B+\frac{1}{2}) \alpha^2 L \sigma^2
    \end{align*}
\end{lemma}
Please refer to Appendix~\ref{sub:case1_proof} for the proof of Lemma~\ref{lem:case1:informal}.

\subsection{In the Case of \texorpdfstring{$\langle v_t, f_t(x_t) \rangle \leq 0$}{}}\label{sub:case2}

We explore the relationship between $\E[f(x_{t+1})]$ and $\E[f(x_t)]$ in the case of $\langle v_t, f_t(x_t) \rangle \leq 0$ below.
\begin{lemma}[Informal version of Lemma~\ref{lem:case2}]\label{lem:case2:informal}
    $f: \R^d \rightarrow \R$ is a $L$-smooth function and has $\sigma$-bounded gradients. Denote $T > 0$ as a positive integer, for a stochastic iterative first-order optimization algorithm that implements $T$ times. For $t \in [T]$, we follows Definition~\ref{def:v:informal} to write it as $x_{t+1} = x_t - \alpha g_t = x_t - \alpha ( \frac{2 |\langle f_t(x_t), v_t \rangle|}{\| v_t \|_2^2} v_t + \nabla f_t(x_t) )$. Let $u_a, u_b > 0$ be defined as Definition~\ref{def:u:informal}. Let $B > 0$ be defined as Definition~\ref{def:B:informal}. If $\langle v_t, f_t(x_t) \rangle \leq 0$, we have
    \begin{align*}
        \E[ f(x_{t+1}) ] \leq & ~ \E[ f(x_t) ] - \alpha ( 1 - 2u_b ) \E[ \| \nabla f(x_t) \|_2^2 ] \\
        & ~ + \frac{1}{2} \alpha^2 L \sigma^2
    \end{align*}
\end{lemma}
Please refer to Appendix~\ref{sub:case2_proof} for the proof of Lemma~\ref{lem:case2:informal}.

\subsection{Main Results}\label{sub:main_result}

Now we provide the main result of our unified framework that can summarize all first-order stochastic methods below.
\begin{theorem}[Informal version of Theorem~\ref{thm:universal_analysis}]\label{thm:universal_analysis:informal}
    Consider a function $f: \R^d \rightarrow \R$ that is $L$-smooth and has $\sigma$-bounded gradients. We start with an initial weight parameter $x_0$. We then apply a first-order stochastic algorithm that updates the weight parameter according to the rule $x_{t+1} = x_t - \alpha g_t$, where $\alpha > 0$ is the learning rate. Denote integer $T > 0$ as the time of iterations. Let integer $k \geq 0$ be defined as Definition~\ref{def:k:informal}, let $l \geq 0$ be defined as Definition~\ref{def:l:informal}. Let constants $u_a, u_b \geq 0$ be defined as Definition~\ref{def:u:informal}. Denote $x^* := \arg \min_x f(x)$. Let $\alpha = c/\sqrt{T+8kB}$ where $c := \sqrt{\frac{2(f(x_0) - f(x^*))}{L \sigma^2}}$, then with running of $T$ iterations, we have
    \begin{align*}
        & ~ \min_{t \in [T]} \E[ \| \nabla f(x_t) \|_2^2 ] \\
        & ~ \leq \frac{\sqrt{T + 8kB}}{T+ 2ku_a - 2lu_b} \sqrt{2(f(x_0) - f(x^*)) \cdot L\sigma^2}
    \end{align*}
\end{theorem}

For the proof of Theorem~\ref{thm:universal_analysis:informal}, please see Appendix~\ref{sub:uni_analysis_main_result}.

\section{Applications of Our Analysis}\label{sec:applications}

This section describes two algorithms that have been developed within our framework, namely Reject Accelerating and Random Vector Accelerating.

In Section~\ref{sub:ra}, we have identified a correlation between inconsistency and convergence rate, as outlined in Theorem~\ref{thm:universal_analysis:informal}. Based on this observation, we propose the Reject Accelerating algorithm, which involves excluding the additional accelerating term $v_t$ when the condition $\langle v_t, f_t(x_t) \rangle \leq 0$ is satisfied. This modification aims to reduce the value of $l$ (as defined in Definition~\ref{def:l:informal}).

Moving on to Section~\ref{sub:rva}, we introduce a technique that involves sampling vectors from a Gaussian distribution $\mathcal{N}(0, \mathbf{I}_d)$ in order to enhance the convergence speed. Additionally, we design the update rule for this algorithm based on the Reject Accelerating scheme proposed in Section~\ref{sub:ra}. 

\subsection{Reject Accelerating}\label{sub:ra}

Following Theorem~\ref{thm:universal_analysis:informal}, which provides the unified convergence rate of first-order algorithms, we have the convergence rate $\frac{\sqrt{T + 8kB}}{T+ 2ku_a - 2lu_b}$, we can easily further derive that:
\begin{align*}
    \frac{\sqrt{T + 8kB}}{T + 2ku_a - 2lu_b} \leq \frac{\sqrt{T + 8kB}}{T+ 2ku_a}
\end{align*}
This inequality holds because the conditions $l \geq 0$ and $u_b > 0$ are satisfied. From a theoretical standpoint, the terms $v_t$ for $t \in [T]$ that satisfy $\langle f(x_t), v_t \rangle$ may have a negative effect on fast convergence.

To address this, we propose a modification to the optimization algorithm that rejects the additional accelerating term $v_t$ and instead uses the original subgradient $\nabla f_t(x_t)$ for updating. This reduces the value of $l$ to 0. We provide a formal definition of this modification below:
\begin{definition}[Reject accelerating]\label{def:reject_accelerating:informal}
    For any first-order algorithm that implements
    \begin{align*}
        x_{t+1} = x_t - \alpha g_t.
    \end{align*}
    We transform it into
    \begin{align*}
        & ~ x_{t+1} \\
        & ~ = \left\{
            \begin{array}{lr}  
             x_t - \alpha g_t, & \langle g_t - \nabla f_t(x_t), \nabla f_t(x_t) \rangle > 0 \\  
             x_t - \alpha \nabla f_t(x_t), & \text{\rm otherwise}
             \end{array} 
        \right.
    \end{align*}
\end{definition}

Then the analysis of our Reject Accelerating method that applies to first-order stochastic optimization gains a better convergence rate, we show our result of analysis as follows:
\begin{theorem}[Informal version of Theorem~\ref{thm:reject_accelerating}]\label{thm:reject_accelerating:informal}
    For any first-order stochastic algorithm that implements $x_{t+1} 
    = x_t - \alpha g_t$ with learning rate $\alpha > 0$ (Definition~\ref{def:reject_accelerating:informal}). For $t \in [T]$ where integer $T > 0$, if we reject to update $x$ by $g_t$ when $\langle g_t - \nabla f_t(x_t), \nabla f_t(x_t) \rangle \leq 0$ but update $x_{t+1} = x_t - \alpha \nabla f_t(x_t)$. Then the convergence rate of this algorithm will be enhanced from $\frac{\sqrt{T + 8kB}}{T+ 2ku_a - 2lu_b}$ to $\frac{\sqrt{T + 8kB}}{T+ 2ku_a}$.
\end{theorem}

We show the proof of Theorem~\ref{thm:reject_accelerating:informal} in Appendix~\ref{sub:ra_main_result}.

\subsection{RVA: Random Gaussian Vector Accelerating}\label{sub:rva}

This section focuses on the convergence of momentum and introduces our approach, RVA (Random Vector Accelerating), which utilizes adaptive Gaussian vectors to achieve a convergence rate superior to that of SGD. We have demonstrated that, under certain conditions, even Gaussian random vectors can accelerate the optimization process. Furthermore, Theorem~\ref{thm:universal_analysis:informal} reveals that increasing the value of $u_a$ and decreasing the value of $B$ are crucial for the success of the first-order acceleration algorithm.

Our method builds upon the SGD and Reject Accelerating techniques discussed in the previous section. Specifically, we consider SGD with $T$ iterations, where for each $t \in [T]$, we introduce a random vector $v_t \in \mathbb{R}^d$ sampled from a Gaussian distribution $\mathcal{N}(0, \mathbf{I}_d)$. In order to achieve a faster convergence rate, as described in Theorem~\ref{thm:reject_accelerating:informal}, we follow the principles outlined in Definition~\ref{def:reject_accelerating:informal} to update the parameters. If $\langle v_t, \nabla f_t(x_t) \rangle > 0$, we update the parameters using $x_{t+1} = x_t - \alpha \left( \nabla f_t(x_t) + \frac{2\langle v_t, \nabla f_t(x_t) \rangle}{\|v_t\|2^2}v_t \right)$. Otherwise, if $\langle v_t, \nabla f_t(x_t) \rangle \geq 0$, we update the parameters using $x_{t+1} = x_t - \alpha \nabla f_t(x_t)$. The detailed algorithm is presented in Algorithm~\ref{alg:SGD_Random_acceleration}.

\input{alg2}

Through our analysis, we demonstrate the following results:
\begin{theorem}[Informal version of Theorem~\ref{thm:convergence_rva}]\label{thm:convergence_rva:informal}
    Consider $f: \R^d \rightarrow \R$ is $L$-smooth and has $\sigma$-bounded gradients. We state with an initial weight parameter $x_0$. Denote positive integer $T \geq 0$. We define $x^* = \min_{x \in \R^d} f(x)$. Let $k$ be defined as Definition~\ref{def:k:informal}. Using Algorithm~\ref{alg:SGD_Random_acceleration} to run $T$ iterations with learning rate $\alpha = c/\sqrt{T + 4\frac{1}{d}k}$ where $c := \sqrt{\frac{2(f(x_0) - f(x^*))}{L\sigma^2}}$, we have
    \begin{align*}
        \min_{t \in [T]}\E[\|f(x)\|_2^2] \leq \frac{\sqrt{T+4\frac{1}{d}k}}{T+2\frac{1}{d}k} \sqrt{2( f(x_0) - f(x^*) ) \cdot L\sigma^2}
    \end{align*}
\end{theorem}

We provide a detailed analysis and proof of Theorem~\ref{thm:convergence_rva:informal} in Appendix~\ref{sub:rva_main_result}.

Based on Theorem~\ref{thm:convergence_rva:informal}, the convergence rate of our RVA method is given by $\frac{\sqrt{T+4\frac{1}{d}k}}{T+2\frac{1}{d}k}$, which is strictly smaller than the convergence rate of SGD ($1/\sqrt{T}$), due to the conditions $k \geq 0$ and $d > 0$. Notably, as $k$ increases, $\frac{\sqrt{T+4\frac{1}{d}k}}{T+2\frac{1}{d}k}$ decreases.

For a more accurate estimation of the convergence rate of RVA, we can further improve the convergence rate to $\frac{\sqrt{T+4k}}{T+2k}$ by setting $d = 1$, effectively treating the entire model element-wise as a set of one-dimensional vectors and utilizing different values of $v_t$ to update its weights. Subsequently, we can estimate using Lemma~\ref{lem:k_estimator:informal} as follows:
\begin{lemma}[Informal version of Lemma~\ref{lem:k_estimator}]\label{lem:k_estimator:informal}
    Consider $v_t \in \R^d$ for $t \in [T]$ where integer $T > 0$ in Algorithm~\ref{alg:SGD_Random_acceleration}, let integer $k > 0$ be defined as Definition~\ref{def:k:informal}, we have
    \begin{itemize}
        \item $\Pr[k \geq T/2] \approx 0.5$.
        \item If $T > 128$, then $\Pr[k > \lfloor T/3 \rfloor] \geq 1 - \frac{1}{T}$.
    \end{itemize}
\end{lemma}

The proof of Lemma~\ref{lem:k_estimator:informal} is showed in Appendix~\ref{sub:Pr_k}.

Hence, there is an approximate probability of 0.5 that $k \geq T/2$. At this point, $\frac{\sqrt{T+4k}}{T+2k} \approx \frac{\sqrt{3}}{2}\frac{1}{\sqrt{T}}$. Additionally, there is an almost certain probability that $k \geq T/3$, at which point $\frac{\sqrt{T+4k}}{T+2k} = \frac{\sqrt{21}}{5}\frac{1}{\sqrt{T}}$. This implies that the convergence speed of RVA is at least $\frac{5}{\sqrt{21}} \approx 1.09$ times that of SGD.

In summary, these findings demonstrate the enhanced convergence rate of the RVA method and its potential to outperform traditional SGD in terms of optimization speed.

%% file: alg2.tex
\begin{algorithm}[!ht]\caption{SGD with Adaptive Random Vector Acceleration}\label{alg:SGD_Random_acceleration}
\begin{algorithmic}[1]
\Statex \textbf{Input: } Function $f: \R^d \rightarrow \R$, parameters $x_0 \in \R^d$, learning rate $\alpha$, number of iterations $T$

\Statex \textbf{Output: } Optimal weight $x^*$

\Procedure{StochasticGradientDescentWithAdaptiveRandomVectorAcceleration}{$f, x, \alpha, T$}


\State $t \leftarrow 0$ \Comment{Initialize $t$}

\While{$t < T$}

\State $g_t \leftarrow \nabla f_{\xi_t}(x_t)$ \Comment{Compute the subgradient using back-propagation}


\State $v_t \sim \mathcal{N}(0, {\bf I}_d)$ \Comment{Sample a vector from Gaussian distribution for acceleration}

\If{$\langle g_t, v_t \rangle > 0$}

\State $v_t \leftarrow \frac{2\alpha | \langle g_t, v_t \rangle |}{\|v_t\|_2^2} v_t$ \Comment{Compute the adaptation term}

\State $x_{t+1} \leftarrow x_{t+1} - \alpha g_t - v_t$ \Comment{Update weight with adaptation term and subgradient}

\Else{}

\State $x_{t+1} \leftarrow x_t - \alpha g_t$ \Comment{Update weight with only subgradient}

\EndIf

\State $t \leftarrow t + 1$

\EndWhile

\State \Return{$x_t$}

\EndProcedure

\end{algorithmic}
\end{algorithm}

%% file: experiments.tex
\begin{figure*}[!ht]
\centering
\includegraphics[width=0.8\textwidth]{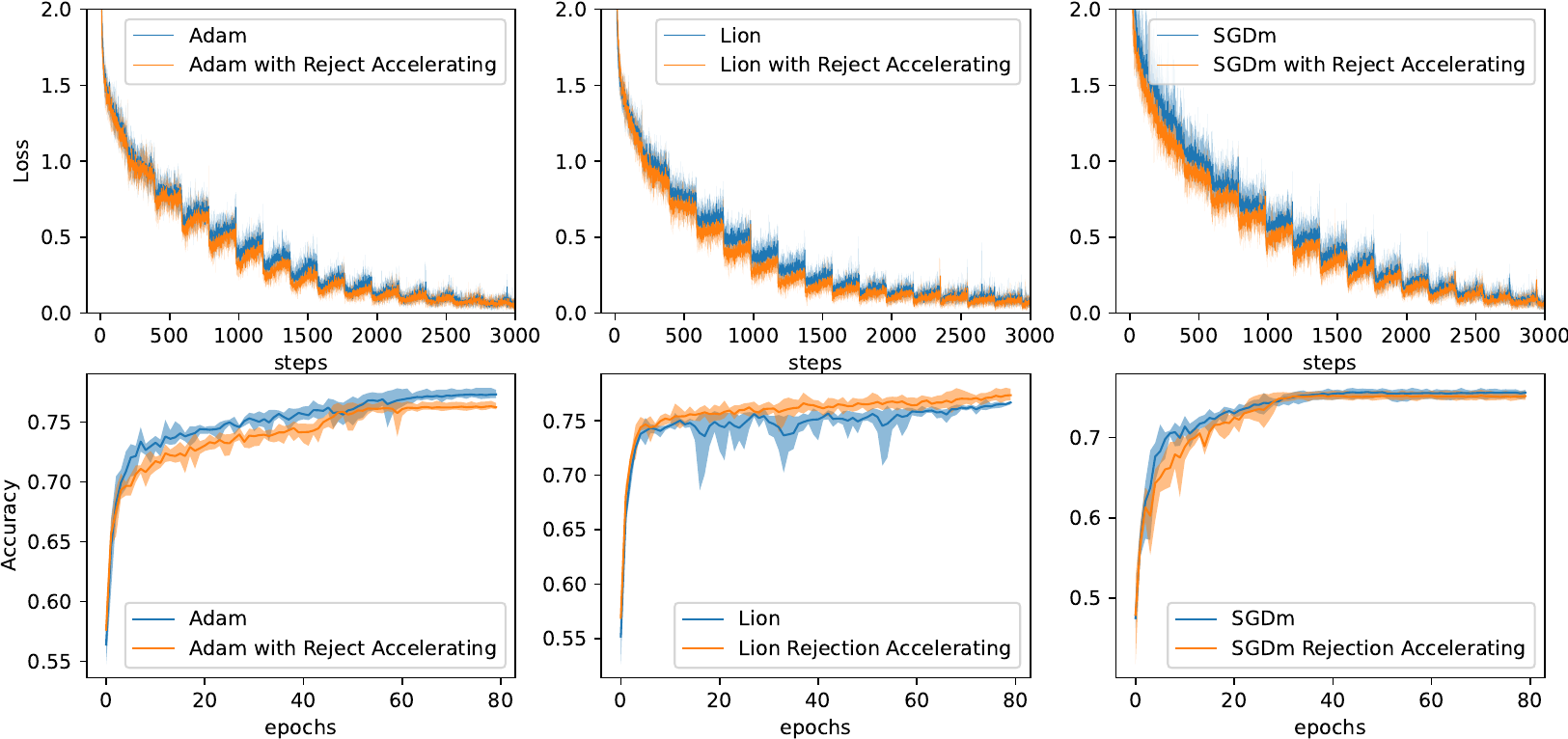}
\caption{Experimental results of applying Reject Accelerating to Adam, Lion, and SGDm optimizers on Cifar-10 dataset.}
\label{fig:ra}
\end{figure*}

\section{Experiments}\label{sec:experiments}

In this section, we aim to assess the efficacy of the methods and their theoretical underpinnings outlined in the preceding section. We delineate our experimental setups in Section~\ref{sub:setup}, and present the results of our evaluation of the Reject Accelerating method in Section~\ref{sub:experiment_ra} to validate our theoretical analysis as articulated in Theorem~\ref{thm:reject_accelerating:informal}. Furthermore, we illustrate the expedited convergence of our Random Vector Accelerating method by applying it to SGD and Adam optimizers in Section~\ref{sub:experiment_rva}.

\subsection{Setup}\label{sub:setup}

\paragraph{Models and Datasets.} For image classification tasks, the Resnet-18 model \cite{hzrs16} is utilized, while the pretrained GPT-2 \cite{rwc+19} is employed for language modeling tasks. The algorithms is assessed using the Cifar-10/100 \cite{kh09} and Penn Treebank \cite{msm93} datasets.

\paragraph{Metrics.} The metric for evaluating the performance in image classification tasks is {\it accuracy} and the metric for evaluating the performance in language modeling tasks is {\it perplexity}.

\paragraph{Hyper-parameters.} For image classification tasks, the number of iterations (epochs) is set to $80$ and the batch size is set to $256$. For language modeling tasks, the number of iterations (epochs) is set to $3$ and the batch size is set to $32$, and we choose the first $256$ tokens in each text of row both in the training set and the test set. During the training, we introduced the linear learning rate scheduler that the learning rate will decrease from the initial value to $0$. All experiments are conducted with random seeds and repeated 5 times to record the maximum, minimum, and average values.

\paragraph{Optimizers.} The following are the basic settings of the optimizers used in our experiment:

{\bf SGD:} The learning rate in image classification tasks is set to $\alpha = {\rm 5e-2}$.

{\bf Adam:} $\beta_1 = 0.9$, $\beta_2 = 0.999$ and $\epsilon = {\rm 1e-8}$. The learning rate in image classification tasks is set to $\alpha = {\rm 1e-3}$ and the learning rate in language modeling tasks is set to $\alpha = {\rm 1e-5}$.

As we apply our RVA method to Adam, we provide the algorithm of Adam with RVA in Appendix~\ref{sec:adam_rva}.

{\bf Lion \cite{clh+23}:} $\beta_1 = 0.9$, $\beta_2 = 0.99$. The learning rate in image classification tasks is set to $\alpha = {\rm 1e-4}$.

{\bf SGDm (SGD with Momentum):} $\mu = 0.8$. The learning rate in image classification tasks is set to $\alpha = {\rm 5e-2}$.

\subsection{Results of Reject Accelerating}
\label{sub:experiment_ra}

Our evaluation focused on the application of the Reject Accelerating method to the Cifar-10 dataset using the Resnet18 architecture model. Specifically, we implemented our method with the Adam, Lion, and SGDm optimizers. The recorded loss and accuracy during training are presented in Figure~\ref{fig:ra} where "step" means the number of updating weights on batch data. Notably, our Reject Accelerating method uniformly enhances training acceleration and improves the performance of the Lion optimizer, while causing minimal impact on the performance of SGDm. However, it is important to highlight that the performance of the Adam optimizer diminishes following the application of Reject Accelerating. This observation may shed light on the underlying disparities in generalization between Adam and SGD, a topic we intend to address in the subsequent section. These findings not only validate our theoretical analysis as articulated in Theorem~\ref{thm:reject_accelerating:informal}, but also underscore the efficacy of our unified framework.

\subsection{Results of Random Vector Accelerating}\label{sub:experiment_rva}

In our assessment of the RVA method, we conducted separate tests on image classification tasks using the Resnet-18 model and language modeling tasks using the GPT-2 model. The training loss records are depicted in Figure~\ref{fig:rva}. The left image in the figure compares the performance of Adam and Adam with RVA on the Penn Treebank dataset. Conversely, the right image presents the loss records for the initial 3000 steps of training Resnet-18 on the Cifar-100 datasets. The results validate the substantial acceleration enhancement of our RVA method when comparing Adam and Adam with RVA. However, only minor improvements were observed on the Cifar-100 dataset. We attribute this to the necessity for an optimal learning rate and gradient variance. In the following section, we will delve into some of the limitations of the RVA method, such as these.

\ifdefined\isarxiv

\begin{figure*}[!ht]
\centering
\includegraphics[width=0.8\textwidth]{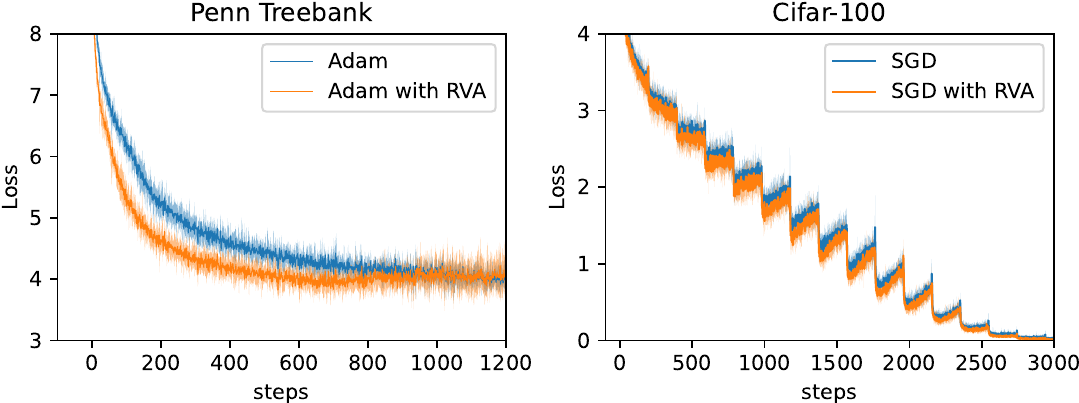}
\caption{Experimental results of applying Random Vector Accelerating to Adam and SGD optimizers on Cifar-100 and Penn Treebank datasets separately.}
\label{fig:rva}
\end{figure*}

\else

\begin{figure}[!ht]
\includegraphics[width=0.48\textwidth]{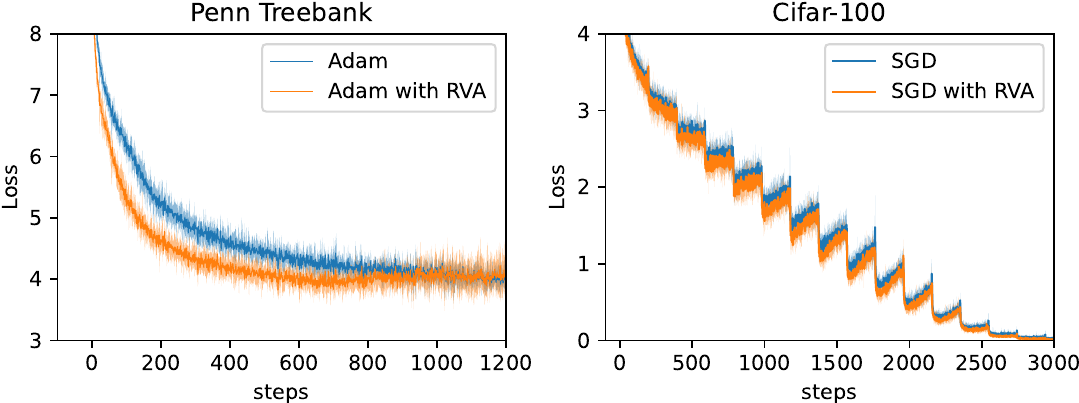}
\caption{Experimental results of applying Random Vector Accelerating to Adam and SGD optimizers on Cifar-100 and Penn Treebank datasets separately.}
\label{fig:rva}
\end{figure}

\fi

In addition, we demonstrate a uniform improvement of performance by RVA in Table~\ref{tab:performance}, where it records the better performance of each method in multiple training.

\ifdefined\isarxiv

\begin{table*}[!ht]
    \centering
    \resizebox{0.8\textwidth}{14mm}{
\begin{tabular}{cccc}
   \toprule
   dataset & optimizer & model & performance \\
   \midrule
   Penn Treebank & Adam & GPT-2 & 39.539 (perplexity) \\
   Penn Treebank & Adam with RVA & GPT-2 & {\bf 35.645} (perplexity) \\
   Cifar-100 & SGD & Resnet-18 & 0.3765 (accuracy) \\
   Cifar-100 & SGD with RVA & Resnet-18 & {\bf 0.4164} (accuracy) \\
   \bottomrule
\end{tabular}
}
    \caption{Comparison of performance of applying RVA
    on SGD and Adam optimizers.}
    \label{tab:performance}
\end{table*}

\else

\begin{table}[!ht]
    \resizebox{0.48\textwidth}{11mm}{
\begin{tabular}{cccc}
   \toprule
   dataset & optimizer & model & performance \\
   \midrule
   Penn Treebank & Adam & GPT-2 & 39.539 (perplexity) \\
   Penn Treebank & Adam with RVA & GPT-2 & {\bf 35.645} (perplexity) \\
   Cifar-100 & SGD & Resnet-18 & 0.3765 (accuracy) \\
   Cifar-100 & SGD with RVA & Resnet-18 & {\bf 0.4164} (accuracy) \\
   \bottomrule
\end{tabular}
}
    \caption{Comparison of performance of applying RVA
    on SGD and Adam optimizers.}
    \label{tab:performance}
\end{table}

\fi

%% file: discussion.tex
\section{Discussion}

While the preceding sections affirm the strength of our methodology, it is important to acknowledge that every approach has its limitations and potential for future exploration. This section is bifurcated into two parts: the first part discusses the constraints of our methodology in Section~\ref{sub:limitations}, and the second part delves into potential future research directions that have emerged from our work in Section~\ref{sub:future_works}.

\subsection{Limitations}\label{sub:limitations}

\paragraph{Cost of Sampling in RVA.}
Although the fast convergence speed of RVA as we proved in Theorem~\ref{thm:convergence_rva:informal}, but in practice, sampling random vectors from Gaussian distribution requires large time consumption, which practically affects the training speed and is hard to employ to optimize large models. 

\paragraph{Requirement of the Optimal Learning Rate.}
From the theoretical perspective, a constant learning rate converges faster than a learning rate with decreased schedule. However, scheduling in learning rate is widely introduced for fast convergence in the early stage of training. Our work lacks of analysis of the case of learning rate with a decreased schedule. Meanwhile, all theorems in this paper require the optimal setting of learning rates, we believe that non-optimal learning rates (especially if the learning rate is too big) may have a negative impact on our accelerating methods.

\paragraph{Effects of Variance in RVA.}
When we face some complex tasks with a variance of subgradients considerably large, stochastic methods might be at a loss for such problems while applying RVA to it will further exacerbate this issue due to RVA enlarging the variance. Even if we ensure the convergence of expectations of gradients, we still don't know will the problem above bring inevitable training collapse.

\subsection{Future Works}\label{sub:future_works}

\paragraph{Inconsistent \texorpdfstring{$v_t$}{} of Adam May be the Key to Generalization.}
In the results of applying Reject Accelerating on Adam in Figure~\ref{fig:ra}, we observe that our method effect negatively to the performance. It inspires us that the inconsistent additional accelerating term $v_t$ of Adam may be the key to generalization and reveal why the Adam optimizer performs better in comparison with SGD.

\paragraph{Advanced RVA by Introducing Momentum as \texorpdfstring{$v_t$}{} instead of Gaussian Vectors.}
One obvious conclusion in our approach is that both increasing $u_a$ and decreasing $B$ can accelerate optimization. Under the premise of Reject Accelerating, even Gaussian vectors can satisfy the conditions for faster convergence. We can easily derive, using momentum instead of Gaussian vectors as $v_t$ may provide us with faster convergence.

%% file: conclusion.tex
\vspace{-2mm}

\section{Conclusion}

This study addresses the complexities of first-order stochastic accelerating methods by introducing a unified framework. This framework interprets method updates as a combination of an auxiliary accelerating term $ v_t $ and the subgradient $ \nabla f_t(x_t) $. This approach simplifies the analysis, akin to examining Stochastic Gradient Descent (SGD), and yields insights into fast optimization. Our research also identified two methods that effectively improve convergence rates, as demonstrated by our analysis. The efficacy of our theoretical framework was validated through extensive experiments across various datasets, models, and optimizers.

%% file: roadmap.tex
\paragraph{Roadmap.} 
We provide the preliminary we use in this paper in Appendix~\ref{sec:preli}. In Appendix~\ref{sec:uni_framework_proof}, we show our proofs of our unified framework in detail. In Appendix~\ref{sec:ra}, we provide the formal proofs of our Reject Accelerating method. Appendix~\ref{sec:rva} shows our analysis of our Random Vector Accelerating method. Finally in Appendix~\ref{sec:adam_rva}, we provide the algorithm we apply our Random Vector Accelerating to Adam optimizer.

%% file: proofs.tex
\section{Preliminary}\label{sec:preli}

\subsection{Notations}\label{sub:notations}

In this paper, we adopt the following notations: The number of dimensions of a neural network is denoted by $d$. The set of real numbers is represented by $\R$. A vector with $d$ elements, where each entry is a real number, is denoted as $x \in \mathbb{R}^{d}$. We use $x$ to denote the weight of a vector parameter in a neural network. For any positive integer $n$, we use $[n]$ to denote $\{1,2,\cdots, n\}$. The $\ell_p$ norm of a vector $x$ is denoted as $\| x \|_p$, for examples, $\| x \|_1 := \sum^n_{i=1} | x_i |$, $\| x \|_2 := ( \sum^n_{i=1} x_i^2 )^{1/2}$ and $\| x \|_\infty := \max_{i \in [n]} | x_i |$. For two vectors $x, y \in \R^d$, we denote $\langle x, y \rangle = \sum^n_{i=1}$ for $i \in [d]$. The loss function of a weight vector parameter $x \in \mathbb{R}^d$ on the entire dataset is denoted as $f: \mathbb{R}^d \rightarrow \mathbb{R}$. We consider the optimization process as consisting of $T$ update steps. For an integer $1 \leq t \leq T$, at step $t$, we have the weight vector parameter $x_t$. We use $\nabla f(x)$ to denote $\frac{\partial f(x)}{\partial x}$. We use $\E[]$ to denote expectation.

\subsection{Definitions}\label{sub:def_ass}




\begin{definition}\label{def:f}
    We define optimizing objective $f: \R^d \rightarrow \R$ with weight of parameters $x \in \R^d$ as follows:
    \begin{align*}
        f(x) = \frac{1}{n} \sum_{i=1}^n f_i(x)
    \end{align*}
    where $f_i(x)$ denotes the loss on batch data.
\end{definition}

\begin{definition}\label{def:lipschitz}
    We say $f: \R^d \rightarrow \R$ is {\it $L$-smooth} such that
\begin{align*}
    \| \nabla f(x) - \nabla f(y) \|_2 \leq L \| x - y \|_2^2, \forall x, y \in \R^d.
\end{align*}
\end{definition}

\begin{definition}\label{def:sigma}
    We say $f: \R^d \rightarrow \R$ has {\it $\sigma$-bounded gradients} if $\| \nabla f_i(x) \|_2 \leq \sigma$, for all $x \in \R^d, i \in [n]$.
\end{definition}

\subsection{Facts}

\begin{fact}\label{fac:lipschitz}
    If $f: \R^d \rightarrow \R$ is {\it $L$-smooth}, then for any $x, y \in \R^d$, there is
    \begin{align*}
        f(y) \leq f(x) + \langle \nabla f(x), y - x \rangle + \frac{L}{2} \| x - y \|_2^2
    \end{align*}
\end{fact}

\begin{fact}\label{fac:unbiased_estimator}
    Let $f$ and $f_i$ be defined as Definition~\ref{def:f}, for $i \in [n]$, $\nabla f_i(x)$ is the unbiased estimator of $\nabla f(x)$ for $x \in \R^d$ such that
    \begin{align*}
        \E[ \nabla f_i(x) ] = \E[ \nabla f(x) ]
    \end{align*}
\end{fact}

\section{Universal Analysis of Accelerating Algorithms}\label{sec:uni_framework_proof}

\subsection{Definitions}

\begin{definition}\label{def:v}
    For any optimization algorithm that can be described as $x_{t+1} = x_t - \alpha g_t$. We define the {\it additional accelerating term} $v_t$ that satisfies $\frac{2 |\langle f_t(x_t), v_t \rangle|}{\| v_t \|_2^2} v_t := g_t - \nabla f_t(x_t)$.
\end{definition}


\begin{definition}\label{def:gamma}
    We define
    \begin{align*}
        & ~ \gamma_k(t) := \left\{
            \begin{array}{lr}  
             1, & \langle v_t, \nabla f_t(x_t) \rangle > 0 \\  
             0, & \text{\rm otherwise}
             \end{array} 
        \right. \\
        & ~ \gamma_l(t) := \left\{
            \begin{array}{lr}  
             1, & \langle v_t, \nabla f_t(x_t) \rangle \leq 0 \textrm{ and } \|v_t\|_2 \ne 0 \\  
             0, & \text{\rm otherwise}
             \end{array} 
        \right.
    \end{align*}
\end{definition}

\begin{definition}\label{def:k}
    For an optimization algorithm that has been run $T$ times where integer $k > 0$, let $\gamma_k(t)$ be defined as Definition~\ref{def:gamma}, we define $k := \sum_{t=0}^{T-1} \gamma_k(t)$.
\end{definition}

\begin{definition}\label{def:l}
    For an optimization algorithm that has been run $T$ times where integer $k > 0$, let $\gamma_l(t)$ be defined as Definition~\ref{def:gamma}, we define $l := \sum_{t=0}^{T-1} \gamma_l(t)$.
\end{definition}

\begin{definition}\label{def:u}
    We define $u_a \E[ \| \nabla f(x_t) \|_2^2 ] \leq \E[ \frac{\langle \nabla f(x_t), v_t \rangle^2}{\| v_t \|_2^2} ] \leq u_b \E[ \| \nabla f(x_t) \|_2^2 ]$ where $u_a, u_b > 0$ are constants.
\end{definition}

\begin{definition}\label{def:B}
    We define $\E[ \frac{\langle \nabla f_t(x_t), v_t \rangle^2}{\| v_t \|_2^2} ] \leq B \E[ \| \nabla f_t(x_t) \|_2^2 ]$ where $B > 0$ is denoted as a constant.
\end{definition}

\subsection{In the Case of \texorpdfstring{$\langle v_t, f_t(x_t) \rangle > 0$}{}}\label{sub:case1_proof}

\begin{lemma}[Formal version of Lemma~\ref{lem:case1:informal}]\label{lem:case1}
    $f: \R^d \rightarrow \R$ is a $L$-smooth function and has $\sigma$-bounded gradients. Denote $T > 0$ as a positive integer, for a stochastic iterative first-order optimization algorithm that implements $T$ times. For $t \in [T]$, we follows Definition~\ref{def:v} to write it as $x_{t+1} = x_t - \alpha g_t = x_t - \alpha ( \frac{2 |\langle f_t(x_t), v_t \rangle|}{\| v_t \|_2^2} v_t + \nabla f_t(x_t) )$. Let $u_a, u_b > 0$ be defined as Definition~\ref{def:u}. Let $B > 0$ be defined as Definition~\ref{def:B}. If $\langle v_t, f_t(x_t) \rangle > 0$, we have
    \begin{align*}
        \E[ f(x_{t+1}) ] \leq \E[ f(x_t) ] - \alpha ( 1 + 2u_a ) \E[ \| \nabla f(x_t) \|_2^2 ] + (4B+\frac{1}{2}) \alpha^2 L \sigma^2
    \end{align*}
\end{lemma}

\begin{proof}
    We have
    \begin{align}\label{eq:lipschitz_nabla_f_t_v_t_positive}
        f(x_{t+1}) 
        \leq & ~ f(x_t) + \langle \nabla f(x_t), x_{t+1} - x_t \rangle + \frac{L}{2} \| x_{t+1} - x_t \|_2^2 \notag \\
        \leq & ~ f(x_t) - \alpha \langle \nabla f(x_t), \frac{2|\langle f_t(x_t), v_t \rangle|}{\| v_t \|_2^2} v_t + \nabla f_t(x_t) \rangle + \frac{\alpha^2 L}{2} \| \frac{2|\langle f_t(x_t), v_t \rangle|}{\| v_t \|_2^2} v_t + \nabla f_t(x_t) \|_2^2 \notag \\
        \leq & ~ f(x_t) - \alpha \langle \nabla f(x_t), \frac{2\langle f_t(x_t), v_t \rangle}{\| v_t \|_2^2} v_t + \nabla f_t(x_t) \rangle + \frac{\alpha^2 L}{2} ( \frac{8\langle f_t(x_t), v_t \rangle^2}{\| v_t \|_2^2} + \| \nabla f_t(x_t) \|_2^2 )
    \end{align}
    where the first step follows from Fact~\ref{fac:lipschitz}, the second step follows from Definition~\ref{def:v}, the last step follows from simple algebra and $\langle f_t(x_t), v_t \rangle > 0$.

    Then we take the expectation, we have
    \begin{align*}
        & ~ \E[f(x_{t+1})] \\
        \leq & ~ \E[ f(x_t) - \alpha \langle \nabla f(x_t), \frac{2\langle f_t(x_t), v_t \rangle}{\| v_t \|_2^2} v_t + \nabla f_t(x_t) \rangle + \frac{\alpha^2 L}{2} ( \frac{8\langle f_t(x_t), v_t \rangle^2}{\| v_t \|_2^2} + \| \nabla f_t(x_t) \|_2^2 ) ] \\
        = & ~ \E[ f(x_t) ] - \E[ \alpha \langle \nabla f(x_t), \frac{2\langle f_t(x_t), v_t \rangle}{\| v_t \|_2^2} v_t + \nabla f_t(x_t) \rangle ] + \E[ \frac{\alpha^2 L}{2} ( \frac{8\langle f_t(x_t), v_t \rangle^2}{\| v_t \|_2^2} + \| \nabla f_t(x_t) \|_2^2 ) ] \\
        = & ~ \E[ f(x_t) ] - \alpha \E[ \langle \nabla f(x_t), \nabla f_t(x_t) \rangle + \langle \nabla f(x_t), \frac{2\langle f_t(x_t), v_t \rangle}{\| v_t \|_2^2} v_t \rangle] \\
        & ~ + \E[ \frac{\alpha^2 L}{2} ( \frac{8\langle f_t(x_t), v_t \rangle^2}{\| v_t \|_2^2} + \| \nabla f_t(x_t) \|_2^2 ) ] \\
        = & ~ \E[ f(x_t) ] - \alpha \E[ \| \nabla f(x_t) \|_2^2 ] - \alpha \E[ \langle \nabla f(x_t), \frac{2\langle f_t(x_t), v_t \rangle}{\| v_t \|_2^2} v_t \rangle] \\
        & ~ + \E[ \frac{\alpha^2 L}{2} ( \frac{8\langle f_t(x_t), v_t \rangle^2}{\| v_t \|_2^2} + \| \nabla f_t(x_t) \|_2^2 ) ] \\
        = & ~ \E[ f(x_t) ] - \alpha \E[ \| \nabla f(x_t) \|_2^2 ] - 2\alpha \E[ \frac{\langle \nabla f(x_t), v_t \rangle \langle f_t(x_t), v_t \rangle}{\| v_t \|_2^2}] \\
        & ~ + \E[ \frac{\alpha^2 L}{2} ( \frac{8\langle f_t(x_t), v_t \rangle^2}{\| v_t \|_2^2} + \| \nabla f_t(x_t) \|_2^2 ) ] \\
        \leq & ~ \E[ f(x_t) ] - \alpha \E[ \| \nabla f(x_t) \|_2^2 ] - 2\alpha u_a \E[ \| \nabla f(x_t) \|_2^2 ] + \E[ \frac{\alpha^2 L}{2} ( \frac{8\langle f_t(x_t), v_t \rangle^2}{\| v_t \|_2^2} + \| \nabla f_t(x_t) \|_2^2 ) ] \\
        = & ~ \E[ f(x_t) ] - \alpha ( 1 + 2u_a ) \E[ \| \nabla f(x_t) \|_2^2 ] + \E[ \frac{\alpha^2 L}{2} ( \frac{8\langle f_t(x_t), v_t \rangle^2}{\| v_t \|_2^2} + \| \nabla f_t(x_t) \|_2^2 ) ] \\
        = & ~ \E[ f(x_t) ] - \alpha ( 1 + 2u_a ) \E[ \| \nabla f(x_t) \|_2^2 ] + 4\alpha^2 L\E[ \frac{\langle f_t(x_t), v_t \rangle^2}{\| v_t \|_2^2} ] + \frac{\alpha^2 L}{2}\E[ \| \nabla f_t(x_t) \|_2^2 ] \\
        \leq & ~ \E[ f(x_t) ] - \alpha ( 1 + 2u_a ) \E[ \| \nabla f(x_t) \|_2^2 ] + 4\alpha^2 L B\E[ \| \nabla f_t(x_t) \|_2^2 ] + \frac{\alpha^2 L}{2}\E[ \| \nabla f_t(x_t) \|_2^2 ] \\
        = & ~ \E[ f(x_t) ] - \alpha ( 1 + 2u_a ) \E[ \| \nabla f(x_t) \|_2^2 ] + (4B+\frac{1}{2}) \alpha^2 L \E[ \| \nabla f_t(x_t) \|_2^2 ] \\
        = & ~ \E[ f(x_t) ] - \alpha ( 1 + 2u_a ) \E[ \| \nabla f(x_t) \|_2^2 ] + (4B+\frac{1}{2}) \alpha^2 L \sigma^2
    \end{align*}
    where the first step follows from Eq.\eqref{eq:lipschitz_nabla_f_t_v_t_positive}, the second step follows from simple expectation rules, the third step follows from simple algebra, the fourth step follows from Fact~\ref{fac:unbiased_estimator}, the fifth step follows from simple algebra, the sixth step follows from Fact~\ref{fac:unbiased_estimator} and Definition~\ref{def:u}, the seventh step follows from simple algebra, the eighth step follows from simple expectation rules, the ninth step follows from Definition~\ref{def:B}, the tenth step follows from simple algebra, the last step follows from Definition~\ref{def:sigma}.
\end{proof}

\subsection{In the Case of \texorpdfstring{$\langle v_t, f_t(x_t) \rangle \leq 0$}{}}\label{sub:case2_proof}

\begin{lemma}[Formal version of Lemma~\ref{lem:case2:informal}]\label{lem:case2}
    $f: \R^d \rightarrow \R$ is a $L$-smooth function and has $\sigma$-bounded gradients. Denote $T > 0$ as a positive integer, for a stochastic iterative first-order optimization algorithm that implements $T$ times. For $t \in [T]$, we follows Definition~\ref{def:v} to write it as $x_{t+1} = x_t - \alpha g_t = x_t - \alpha ( \frac{2 |\langle f_t(x_t), v_t \rangle|}{\| v_t \|_2^2} v_t + \nabla f_t(x_t) )$. Let $u_a, u_b > 0$ be defined as Definition~\ref{def:u}. Let $B > 0$ be defined as Definition~\ref{def:B}. If $\langle v_t, f_t(x_t) \rangle \leq 0$, we have
    \begin{align*}
        \E[ f(x_{t+1}) ] \leq \E[ f(x_t) ] - \alpha ( 1 - 2u_b ) \E[ \| \nabla f(x_t) \|_2^2 ] + \frac{1}{2} \alpha^2 L \sigma^2
    \end{align*}
\end{lemma}

\begin{proof}
    We have
    \begin{align}\label{eq:lipschitz_nabla_f_t_v_t_negative}
        f(x_{t+1}) 
        \leq & ~ f(x_t) + \langle \nabla f(x_t), x_{t+1} - x_t \rangle + \frac{L}{2} \| x_{t+1} - x_t \|_2^2 \notag \\
        \leq & ~ f(x_t) - \alpha \langle \nabla f(x_t), \frac{2|\langle f_t(x_t), v_t \rangle|}{\| v_t \|_2^2} v_t + \nabla f_t(x_t) \rangle + \frac{\alpha^2 L}{2} \| \frac{2|\langle f_t(x_t), v_t \rangle|}{\| v_t \|_2^2} v_t + \nabla f_t(x_t) \|_2^2 \notag \\
        \leq & ~ f(x_t) - \alpha \langle \nabla f(x_t), \frac{2|\langle f_t(x_t), v_t \rangle|}{\| v_t \|_2^2} v_t + \nabla f_t(x_t) \rangle \notag \\
        & ~ + \frac{\alpha^2 L}{2} ( \frac{4\langle f_t(x_t), v_t \rangle^2}{\| v_t \|_2^2} - \frac{4\langle f_t(x_t), v_t \rangle^2}{\| v_t \|_2^2} + \| \nabla f_t(x_t) \|_2^2 ) \notag \\
        \leq & ~ f(x_t) - \alpha \langle \nabla f(x_t), \frac{2|\langle f_t(x_t), v_t \rangle|}{\| v_t \|_2^2} v_t + \nabla f_t(x_t) \rangle + \frac{\alpha^2 L}{2} \| \nabla f_t(x_t) \|_2^2 
    \end{align}
    where the first step follows from Fact~\ref{fac:lipschitz}, the second step follows from Definition~\ref{def:v}, the third step follows from simple algebra and $\langle f_t(x_t), v_t \rangle \leq 0$, the last step follows from simple algebra.

    Then we take the expectation, we have
    \begin{align*}
        & ~ \E[f(x_{t+1})] \\
        \leq & ~ \E[ f(x_t) - \alpha \langle \nabla f(x_t), \frac{2|\langle f_t(x_t), v_t \rangle|}{\| v_t \|_2^2} v_t + \nabla f_t(x_t) \rangle + \frac{\alpha^2 L}{2} \| \nabla f_t(x_t) \|_2^2  ] \\
        = & ~ \E[ f(x_t) ] - \E[ \alpha \langle \nabla f(x_t), \frac{2|\langle f_t(x_t), v_t \rangle|}{\| v_t \|_2^2} v_t + \nabla f_t(x_t) \rangle ] + \E[ \frac{\alpha^2 L}{2} \| \nabla f_t(x_t) \|_2^2  ] \\
        = & ~ \E[ f(x_t) ] - \alpha \E[ \langle \nabla f(x_t), \nabla f_t(x_t) \rangle ] - \alpha \E[ \langle \nabla f(x_t), \frac{2|\langle f_t(x_t), v_t \rangle|}{\| v_t \|_2^2} v_t \rangle ] + \E[ \frac{\alpha^2 L}{2} \| \nabla f_t(x_t) \|_2^2  ] \\
        = & ~ \E[ f(x_t) ] - \alpha \E[ \| \nabla f(x_t) \|_2^2 ] - \alpha \E[ \langle \nabla f(x_t), \frac{2|\langle f_t(x_t), v_t \rangle|}{\| v_t \|_2^2} v_t \rangle ] + \E[ \frac{\alpha^2 L}{2} \| \nabla f_t(x_t) \|_2^2  ] \\
        = & ~ \E[ f(x_t) ] - \alpha \E[ \| \nabla f(x_t) \|_2^2 ] + \alpha \E[ \frac{2\langle \nabla f(x_t), v_t \rangle \langle f_t(x_t), v_t \rangle}{\| v_t \|_2^2} ] + \E[ \frac{\alpha^2 L}{2} \| \nabla f_t(x_t) \|_2^2  ] \\
        = & ~ \E[ f(x_t) ] - \alpha \E[ \| \nabla f(x_t) \|_2^2 ] + \alpha \E[ \frac{2\langle \nabla f(x_t), v_t \rangle^2}{\| v_t \|_2^2} ] + \E[ \frac{\alpha^2 L}{2} \| \nabla f_t(x_t) \|_2^2  ] \\
        \leq & ~ \E[ f(x_t) ] - \alpha \E[ \| \nabla f(x_t) \|_2^2 ] + 2\alpha u_b \E[ \| \nabla f(x_t) \|_2^2 ] + \E[ \frac{\alpha^2 L}{2} \| \nabla f_t(x_t) \|_2^2  ] \\
        \leq & ~ \E[ f(x_t) ] - \alpha (1 - 2 u_b) \E[ \| \nabla f(x_t) \|_2^2 ] + \E[ \frac{\alpha^2 L}{2} \| \nabla f_t(x_t) \|_2^2  ] \\
        \leq & ~ \E[ f(x_t) ] - \alpha (1 - 2 u_b) \E[ \| \nabla f(x_t) \|_2^2 ] + \frac{\alpha^2 L}{2} \sigma^2
    \end{align*}
    where the first step follows from Eq.~\ref{eq:lipschitz_nabla_f_t_v_t_negative}, the second and third steps follow from simple expectation rules and simple algebras, the fourth step follows from Fact~\ref{fac:unbiased_estimator}, the fifth step follows from $\langle \nabla f_t(x_t), v_t \rangle \leq 0$, the sixth step follows from Fact~\ref{fac:unbiased_estimator}, the seventh step follows from Definition~\ref{def:u}, the eighth step follows from simple algebra, the last step follows from Definition~\ref{def:sigma}.
\end{proof}

\subsection{Main Results}\label{sub:uni_analysis_main_result}

\begin{theorem}[Formal version of Theorem~\ref{thm:universal_analysis:informal}]\label{thm:universal_analysis}
    Consider a function $f: \R^d \rightarrow \R$ that is $L$-smooth and has $\sigma$-bounded gradients. We start with an initial weight parameter $x_0$. We then apply a first-order stochastic algorithm that updates the weight parameter according to the rule $x_{t+1} = x_t - \alpha g_t$, where $\alpha > 0$ is the learning rate. Denote integer $T > 0$ as the time of iterations. Let integer $k \geq 0$ be defined as Definition~\ref{def:k}, let $l \geq 0$ be defined as Definition~\ref{def:l}. Let constants $u_a, u_b \geq 0$ be defined as Definition~\ref{def:u}. Denote $x^* := \arg \min_x f(x)$. Let $\alpha = c/\sqrt{T+8kB}$ where $c := \sqrt{\frac{2(f(x_0) - f(x^*))}{L \sigma^2}}$, then with running of $T$ iterations, we have
    \begin{align*}
        \min_{t \in [T]} \E[ \| \nabla f(x_t) \|_2^2 ] \leq \frac{\sqrt{T + 8kB}}{T+ 2ku_a - 2lu_b} \sqrt{2(f(x_0) - f(x^*)) \cdot L\sigma^2}
    \end{align*}
\end{theorem}

\begin{proof}
    We have
    \begin{align*}
        & ~ \min_{t \in [T]} \E[ \| \nabla f(x_t) \|_2^2 ] \\
        = & ~ \frac{\alpha( k(1 + 2u_a) + l(1 - 2u_b) ) \min_{t \in [T]} \E[ \| \nabla f(x_t) \|_2^2 ]}{\alpha( k(1 + 2u_a) + l(1 - 2u_b) )} \\
        \leq & ~ \frac{\sum_{t=1, \langle \nabla f_t(x_t), v_t \rangle > 0}^T \alpha(1 + 2u_a)\E[\|\nabla f(x_t)\|_2^2] + \sum_{t=1, \langle \nabla f_t(x_t), v_t \rangle \leq 0}^T \alpha(1 - 2u_b)\E[\|\nabla f(x_t)\|_2^2]}{\alpha( k(1 + 2u_a) + l(1 - 2u_b) )} \\
        \leq & ~ \frac{\sum_{t=0}^{T-1} (\E[f(x_t)] - \E[f(x_{t+1})] + \frac{1}{2} \alpha^2 L \sigma^2) + \sum_{t=1, \langle \nabla f_t(x_t), v_t \rangle > 0}^T 4 B \alpha^2 L \sigma^2 }{\alpha( k(1 + 2u_a) + l(1 - 2u_b) )} \\
        \leq & ~ \frac{\E[f(x_0)] - \E[f(x_{T})] + \frac{1}{2} T \alpha^2 L \sigma^2 + 4 k B \alpha^2 L \sigma^2 }{\alpha( k(1 + 2u_a) + l(1 - 2u_b) )} \\
        \leq & ~ \frac{\E[f(x_0)] - \E[f(x^*)] + \frac{1}{2} T \alpha^2 L \sigma^2 + 4 k B \alpha^2 L \sigma^2 }{\alpha( k(1 + 2u_a) + l(1 - 2u_b) )} \\
        \leq & ~ \frac{f(x_0) - f(x^*) + \frac{1}{2} T \alpha^2 L \sigma^2 + 4 k B \alpha^2 L \sigma^2 }{\alpha( k(1 + 2u_a) + l(1 - 2u_b) )} \\
        \leq & ~ \frac{f(x_0) - f(x^*) + \frac{1}{2} T \alpha^2 L \sigma^2 + 4 k B \alpha^2 L \sigma^2 }{\alpha (T + 2ku_a - 2lu_b)} \\
        \leq & ~ \frac{f(x_0) - f(x^*) + (\frac{1}{2} T + 4kB) \alpha^2 L \sigma^2 }{\alpha (T + 2ku_a - 2lu_b)} \\
        \leq & ~ \frac{\sqrt{T + 8kB}}{T+ 2ku_a - 2lu_b} \sqrt{2(f(x_0) - f(x^*)) \cdot L\sigma^2}
    \end{align*}
    where the first step follows from simple algebra, the second step follows from $\min_{t \in [T]} \E[ \| \nabla f(x_t) \|_2^2 ]$ is the minimum, the third step follows from Lemma~\ref{lem:case1} and Lemma~\ref{lem:case2}, the fourth step follows from simple algebra and $\sum_{t=0}^{T-1} (\E[f(x_t)] - \E[f(x_{t+1})]) = \E[f(x_0)] - \E[f(x_T)]$, the fifth step follows from $f(x^*) \leq f(x_T)$, the sixth step follows from $f(x_0)$ and $f(x^*)$ are fixed values, the seventh and eighth steps follow from $T = k+l$ and simple algebras, the last step follows from $\alpha = c / \sqrt{T + 8kB}$ where 
    \begin{align*}
        c := \sqrt{\frac{2(f(x_0) - f(x^*))}{L \sigma^2}}.
    \end{align*}
\end{proof}

\section{Reject Accelerating}\label{sec:ra}

\subsection{Definition}

\begin{definition}[Reject Accelerating]\label{def:reject_accelerating}
    For any first-order algorithm that implements 
    \begin{align*}
        x_{t+1} = x_t - \alpha g_t.
    \end{align*}
    We transform it into
    \begin{align*}
        & ~ x_{t+1} \\
        & ~ = \left\{
            \begin{array}{lr}  
             x_t - \alpha g_t, & \langle g_t - \nabla f_t(x_t), \nabla f_t(x_t) \rangle > 0 \\  
             x_t - \alpha \nabla f_t(x_t), & \text{\rm otherwise}
             \end{array} 
        \right.
    \end{align*}
\end{definition}

\subsection{Main Result}\label{sub:ra_main_result}

\begin{theorem}[Formal version of Theorem~\ref{thm:reject_accelerating:informal}]\label{thm:reject_accelerating}
    For any first-order stochastic algorithm that implements $x_{t+1} 
    = x_t - \alpha g_t$ with learning rate $\alpha > 0$ (Definition~\ref{def:reject_accelerating}). For $t \in [T]$ where integer $T > 0$, if we reject to update $x$ by $g_t$ when $\langle g_t - \nabla f_t(x_t), \nabla f_t(x_t) \rangle \leq 0$ but update $x_{t+1} = x_t - \alpha \nabla f_t(x_t)$. Then the convergence rate of this algorithm will be enhanced from $\frac{\sqrt{T + 8kB}}{T+ 2ku_a - 2lu_b}$ to $\frac{\sqrt{T + 8kB}}{T+ 2ku_a}$.
\end{theorem}

\begin{proof}
    For $t \in [T]$ where integer $T > 0$, if we reject to update $x$ by $g_t$ when $\langle g_t - \nabla f_t(x_t), \nabla f_t(x_t) \rangle \leq 0$ but update $x_{t+1} = x_t - \alpha \nabla f_t(x_t)$. We have
    \begin{align*}
        l = & ~ \sum_{t=0}^{T-1} \gamma_l(t) \\
        = & ~ 0
    \end{align*}
    where this step follows from all $v_t$ for $0 \leq t \leq T-1$ satisfies $\|v_t\|_2=0$ or $\langle v_t, \nabla f_t(x_t) \rangle > 0$.

    Thus, we have
    \begin{align*}
        \min_{t \in [T]} \E[ \| \nabla f(x_t) \|_2^2 ] 
        \leq & ~ \frac{\sqrt{T + 8kB}}{T+ 2ku_a - 2lu_b} \sqrt{2(f(x_0) - f(x^*)) \cdot L\sigma^2} \\
        = & ~ \frac{\sqrt{T + 8kB}}{T+ 2ku_a} \sqrt{2(f(x_0) - f(x^*)) \cdot L\sigma^2} 
    \end{align*}
    where the first step follows from Theorem~\ref{thm:universal_analysis}, the second step follows from $l=0$.
\end{proof}

\section{Random Vector Accelerating}\label{sec:rva}



\subsection{Helpful Expectation Lemmas}

\begin{lemma}\label{lem:helpful_expectations}
    For a vector $x \in \R^d$ and a random vector $y \in \R^d$ that is sampled from $\mathcal{N}(0, b\cdot\mathbf{I}_d)$, we have
    \begin{itemize}
        \item {\bf Part 1.} For $i \in [d]$, 
        \begin{align*}
            \E[y_i^2] = b 
        \end{align*}
        \item {\bf Part 2.}
        \begin{align*}
            \E[\langle x, y \rangle^2] = b\E[\|x\|_2^2]
        \end{align*}
        \item {\bf Part 3.} For $i, j \in [d]$ and $i \ne j$
        \begin{align*}
            \E[\frac{y_i y_j}{\sum_{k=1}^d y_k^2}] = 0
        \end{align*}
        \item {\bf Part 4.} For $i \in [d]$
        \begin{align*}
            \E[\frac{y_i^2}{\sum_{k=1}^d y_k}] = \frac{1}{d}
        \end{align*}
        \item {\bf Part 5.}
        \begin{align*}
            \E[\frac{\langle x, y \rangle^2}{\|y\|_2^2}] = \frac{1}{d} \E[\|x\|_2^2]
        \end{align*}
        \item {\bf Part 6.} 
        \begin{align*}
            \E[\frac{\langle x, y \rangle^2}{\| y \|_2^2} ~|~ \langle x, y \rangle >0 ] = \frac{1}{2d} \E[\|x\|_2^2]
        \end{align*}
    \end{itemize}
\end{lemma}

\begin{proof}
    {\bf Proof of Part 1.}
    For $i \in [d]$, we have
    \begin{align*}
        \E[y_i^2] = & ~ \E[y_i]^2 + \mathbf{Var}[y] \\
        = & ~ b
    \end{align*}
    where the first step follows from simple expectation rule, the second step follows from $y \sim \mathcal{N}(0, b)$.

    {\bf Proof of Part 2.}
    We have
    \begin{align*}
        \E[\langle x, y \rangle^2]
        = & ~ \E[(\sum_{i=1}^d x_i y_i)^2] \\
        = & ~ \E[\sum_{i=1}^d x_i y_i (\sum_{j=1}^d x_j y_j) ] \\
        = & ~ \E[\sum_{i=1}^d x_i^2 y_i^2 + 2\sum_{i=1}^{d-1} \sum_{j=i+1}^d x_i x_j y_i y_j] \\
        = & ~ \E[\sum_{i=1}^d x_i^2 y_i^2] \\
        = & ~ \sum_{i=1}^d \E[x_i^2] \E[y_i^2] \\
        = & ~ b \E[\|x\|_2^2]
    \end{align*}
    where the first and second steps follow from simple algebras, the third step follows from $\E[y_i] = 0$ and each $y_i$ is independent, the fourth step follows from $x$ and $y$ is independent, the last step follows from $\sum_{i=1}^d x_i^2 = \| x \|_2^2$ and Part 1.







{\bf Proof of Part 3.}
For convenience, we denote $z = \sum_{k=1}^d y_k^2$. Next, $C_0, C_1, ..., C_{d-2}$ are all denoted as constants. We define $\mathrm{Ei}(x) = \int \frac{\exp(x)}{x} \d x$. We have 
\begin{align*}
    & ~ \E[\frac{y_i y_j}{\sum_{k=1}^d y_k^2}] \\
    = & ~ \overbrace{ \int_{-\infty}^{+\infty} \int_{-\infty}^{+\infty} \dots \int_{-\infty}^{+\infty} }^{d} \frac{y_i y_j}{\sum_{k=1}^d y_k^2} \cdot \prod_{k=1}^d \frac{\exp( -\frac{y_k^2}{2} )}{\sqrt{2\pi}} \prod_{k=1}^d \d y_k \\
    = & ~ \frac{1}{(2\pi)^{d/2}}\overbrace{ \int_{-\infty}^{+\infty} \int_{-\infty}^{+\infty} \dots \int_{-\infty}^{+\infty} }^{d-2} ( 0.5z \cdot \mathrm{Ei}(0.5 z) + \exp(0.5 z) + C_0 ) \bigg|_{y_i=-\infty}^{y_i=+\infty} \bigg|_{y_j=-\infty}^{y_j=+\infty} \prod_{k=1, k\ne i, j}^d \d y_k \\
    = & ~ \frac{1}{(2\pi)^{d/2}}\overbrace{ \int_{-\infty}^{+\infty} \int_{-\infty}^{+\infty} \dots \int_{-\infty}^{+\infty} }^{d-2} 0 \prod_{k=1, k\ne i, j}^d \d y_k \\
    = & ~ \frac{1}{(2\pi)^{d/2}}\overbrace{ \int_{-\infty}^{+\infty} \int_{-\infty}^{+\infty} \dots \int_{-\infty}^{+\infty} }^{d-3} C_1 \bigg|_{y_{k_1}=-\infty}^{y_{k_1}=+\infty} \prod_{k=1, k\ne i, j, k_1}^d \d y_k \\
    & ~ \cdots \\
    = & ~ \frac{1}{(2\pi)^{d/2}} \int_{-\infty}^{+\infty} C_{d-3} \bigg|_{y_{k_{d-3}}=-\infty}^{y_{k_{d-3}}=+\infty} \d y_{k_{d-2}} \\
    = & ~ \frac{1}{(2\pi)^{d/2}} \cdot C_{d-2} \bigg|_{y_{k_{d-2}}=-\infty}^{y_{k_{d-2}}=+\infty} \\
    = & ~ 0
\end{align*}

{\bf Proof of Part 4.}
For $i, j \in [d]$, it's easy to have
\begin{align}\label{eq:E_frac_y_i}
    \E[\frac{y_i^2}{\sum_{k=1}^dy_k}] = \E[\frac{y_j^2}{\sum_{k=1}^dy_k}]
\end{align}
Then we have
\begin{align*}
    \E[\frac{y_i^2}{\sum_{k=1}^dy_k}]
    = &  ~ \frac{1}{d} d \cdot \E[\frac{y_i^2}{\sum_{k=1}^dy_k}] \\
    = & ~ \frac{1}{d} \sum_{j=1}^d \E[\frac{y_j^2}{\sum_{k=1}^dy_k}] \\
    = & ~ \frac{1}{d}
\end{align*}
where the first step follows from simple algebra, the second step follows from Eq.\eqref{eq:E_frac_y_i}, the last step follows from simple algebra.

{\bf Proof of Part 5.}
We have
\begin{align*}
    \E[\frac{\langle x, y \rangle^2}{\| y \|_2^2}]
    = & ~ \E[\frac{(\sum_{i=1}^d x_i y_i)^2}{\sum_{k=1}^d y_k^2}] \\
    = & ~ \E[\frac{\sum_{i=1}^d x_i^2 y_i^2 + 2\sum_{i=1}^{d-1} \sum_{j=i+1}^d x_i x_j y_i y_j}{\sum_{k=1}^d y_k^2}] \\
    = & ~ \sum_{i=1}^d \E[\frac{x_i^2 y_i^2}{\sum_{k=1}^d y_k^2}] + 2\sum_{i=1}^{d-1} \sum_{j=i+1}^d \E[\frac{x_i x_j y_i y_j}{\sum_{k=1}^d y_k^2}] \\
    = & ~ \sum_{i=1}^d \E[\frac{x_i^2 y_i^2}{\sum_{k=1}^d y_k^2}] \\
    = & ~ \sum_{i=1}^d x_i^2 \frac{1}{d} \\
    = & ~ \E[\frac{1}{d} \| x \|_2^2]
\end{align*}
where the first, second and third step follows from simple algebras, the fourth step follows from Part 3, the fifth step follows from Part 4, the last step follows from the definition of $\ell_2$ norm.

{\bf Proof of Part 6.}
By Part 5, and the symmetry of the distribution of $y$ and $\langle x, y \rangle$, it's easily to get 
\begin{align*}
    \E[\frac{\langle x, y \rangle^2}{\| y \|_2^2} ~|~ \langle x, y \rangle >0 ] = \frac{1}{2}\E[\frac{\langle x, y \rangle^2}{\| y \|_2^2}] = \frac{1}{2d}\E[\|x\|_2^2]
\end{align*}
    
\end{proof}

\subsection{Probability of \texorpdfstring{$k$}{}}\label{sub:Pr_k}

\begin{lemma}[Formal version of Lemma~\ref{lem:k_estimator:informal}]\label{lem:k_estimator}
    Consider $v_t \in \R^d$ for $t \in [T]$ where integer $T > 0$ in Algorithm~\ref{alg:SGD_Random_acceleration}, let integer $k > 0$ be defined as Definition~\ref{def:k}, we have
    \begin{itemize}
        \item $\Pr[k \geq T/2] \approx 0.5$.
        \item If $T > 128$, then $\Pr[k > \lfloor T/3 \rfloor] \geq 1 - \frac{1}{T}$.
    \end{itemize}
\end{lemma}

\begin{proof}
    It's easy to have
    \begin{align*}
        \Pr[\langle v_t, \nabla f_t(x_t) \rangle > 0] = 0.5
    \end{align*}
    then we can get
    \begin{align}\label{eq:Pr_k}
        \Pr[k = a] = \frac{T!}{(T-a)!a!}0.5^T
    \end{align}
    where $a > 0$ is a scalar.

    {\bf Proof of Part 1.}
    We have
    \begin{align*}
        \Pr[k \geq T/2] = \sum_{i=1}^{T/2} \frac{T!}{(T-i)!i!}0.5^T
    \end{align*}
    where this step follows from Eq.\eqref{eq:Pr_k}.

    Since $\Pr[k \geq T/2]$ is symmetric, we have
    \begin{align*}
        \Pr[k \geq T/2] \approx \Pr[k < T/2] = 1 / 2
    \end{align*}

    {\bf Proof of Part 2.}
    We denote $u = \lfloor T/3 \rfloor$
    \begin{align*}
        \Pr[k > u] = & ~ \sum_{i=1}^{u} \frac{T!}{(T-i)!i!}0.5^T \\
        \leq & ~ u \frac{T!}{(T-u)!u!}0.5^T \\
        \leq & ~ 1 - \frac{1}{T}
    \end{align*}
    where the first step follows from Eq.\eqref{eq:Pr_k}, the second step follows from $\frac{\d \frac{T!}{(T-i)!i!}0.5^T}{\d i} < 0$, the last step follows from 
    \begin{align*}
        \frac{\d }{\d u} ( u \frac{T!}{(T-u)!u!}0.5^T - 1 + \frac{1}{T} ) < 0
    \end{align*}
    and
    \begin{align*}
        u \frac{T!}{(T-u)!u!}0.5^T < 1 - \frac{1}{T}
    \end{align*}
    when $u > 24$.
\end{proof}

\subsection{Main Results}\label{sub:rva_main_result}

\begin{theorem}[Formal version of Theorem~\ref{thm:convergence_rva:informal}]\label{thm:convergence_rva}
    Consider $f: \R^d \rightarrow \R$ is $L$-smooth and has $\sigma$-bounded gradients. We state with an initial weight parameter $x_0$. Denote positive integer $T \geq 0$. We define $x^* = \min_{x \in \R^d} f(x)$. Let $k$ be defined as Definition~\ref{def:k}. Using Algorithm~\ref{alg:SGD_Random_acceleration} to run $T$ iterations with learning rate $\alpha = c/\sqrt{T + 4\frac{1}{d}k}$ where $c := \sqrt{\frac{2(f(x_0) - f(x^*))}{L\sigma^2}}$, we have
    \begin{align*}
        \min_{t \in [T]}\E[\|f(x)\|_2^2] \leq \frac{\sqrt{T+4\frac{1}{d}k}}{T+2\frac{1}{d}k} \sqrt{2( f(x_0) - f(x^*) ) \cdot L\sigma^2}
    \end{align*}
\end{theorem}

\begin{proof}
    This proof follows from Theorem~\ref{thm:universal_analysis}, Part 5 and Part 6 of Lemma~\ref{lem:helpful_expectations}.
\end{proof}

%% file: Adam_RVA.tex
\section{Adam with Random Vector Accelerating}\label{sec:adam_rva}

\begin{algorithm}[!ht]\caption{Adam with Adaptive Random Vector Acceleration}\label{alg:Adam_rva}
\begin{algorithmic}[1]
\Statex \textbf{Input: } Function $f: \R^d \rightarrow \R$, parameters $x_0 \in \R^d$, learning rate $\alpha$, exponential decay rates for the moment estimates $\beta_1, \beta_2$, number of iterations $T$, machine precision threshold $\epsilon$

\Statex \textbf{Output: } Optimal weight $x^*$

\Procedure{AdamWithAdaptiveRandomVectorAcceleration}{$f, x, \alpha, (\beta_1, \beta_2), T, \epsilon$}


\State $t \leftarrow 0$ \Comment{Initialize $t$}

\State $m_0 \leftarrow 0$ \Comment{Initialize 1st moment vector}

\State $v_0 \leftarrow 0$ \Comment{Initialize 2nd moment vector}

\While{$t < T$}

\State $t \leftarrow t + 1$

\State $g_t \leftarrow \nabla f_{t-1}(x_{t-1})$ \Comment{Compute the subgradient using back-propagation}

\State $m_t \leftarrow \beta_1 m_{t-1} + (1 - \beta_1) g_t$ \Comment{Update biased first moment estimate}

\State $v_t \leftarrow \beta_2 v_{t-1} + (1 - \beta_2) g_t^2$ \Comment{Update biased second raw moment estimate}

\State $\hat{m}_t \leftarrow m_t / (1 - \beta_1^t)$ \Comment{Compute bias-corrected first moment estimate}

\State $\hat{v}_t \leftarrow v_t / (1 - \beta_2^t)$ \Comment{Compute bias-corrected second raw moment estimate}

\State $g_t \leftarrow \frac{\hat{m}_t}{\sqrt{\hat{v}_t}+\epsilon}$ \Comment{Compute the update of Adam}

\State $u_t \sim \mathcal{N}(0, {\bf I}_d)$ \Comment{Sample a vector from Gaussian distribution for acceleration}

\If{$\langle g_t, u_t \rangle > 0$}

\State $v_t \leftarrow \frac{2\alpha | \langle g_t, u_t \rangle |}{\|u_t\|_2^2} u_t$ \Comment{Compute the adaptation term}

\State $x_{t+1} \leftarrow x_{t+1} - \alpha g_t - u_t$ \Comment{Update weight with adaptation term and subgradient}

\Else{}

\State $x_{t+1} \leftarrow x_t - \alpha g_t$ \Comment{Update weight with only subgradient}

\EndIf

\EndWhile

\State \Return{$x_t$}

\EndProcedure

\end{algorithmic}
\end{algorithm}